\documentclass[letterpaper,10pt,conference]{ieeeconf}

\IEEEoverridecommandlockouts
\usepackage{amsmath,amssymb,amsfonts,mathtools}
\usepackage{graphicx}
\usepackage{booktabs}
\usepackage{multirow}
\usepackage{array}
\usepackage{bm}
\usepackage{tikz}
\usepackage{algorithm}
\usepackage{algorithmic}
\usepackage{cite}
\usepackage{xcolor}

\makeatletter
\let\NAT@parse\undefined
\makeatother

\usepackage[
    colorlinks=true,
    citecolor=blue,
    linkcolor=blue,
    urlcolor=blue
]{hyperref}

\newtheorem{lemma}{Lemma}
\newtheorem{theorem}{Theorem}
\newtheorem{proposition}{Proposition}
\newtheorem{remark}{Remark}

\DeclareMathOperator{\Tr}{Tr}

\title{\LARGE \bf
An Explicit Surrogate for Gaussian Mixture Flow Matching with Wasserstein Gap Bounds
}

\author{
Elham Rostami$^{1}$, Taous-Meriem Laleg-Kirati$^{1}$, and
Hamidou Tembine$^{2}$ \\[2mm]
$^{1}$ Université Paris-Saclay, Inria, CIAMS, Gif-sur-Yvette, 91190, France\\
\texttt{elham.rostami@universite-paris-saclay.fr}\\[1mm]
\texttt{taous-meriem.laleg-kirati@inria.fr}\\[1mm]
$^{2}$ Department of Electrical Engineering and Computer Science,\\ 
University of Quebec in Trois-Rivieres (UQTR), Quebec, Canada\\
\texttt{tembine@ieee.org}
}

\begin{document}
\maketitle

\begin{abstract}
We study training-free flow matching between two Gaussian mixture models (GMMs) using explicit velocity fields that transport one mixture into the other over time. Our baseline approach constructs component-wise Gaussian paths with affine velocity fields satisfying the continuity equation, which yields to a closed-form surrogate for the pairwise kinetic transport cost. In contrast to the exact Gaussian Wasserstein cost, which relies on matrix square-root computations, the surrogate admits a simple analytic expression derived from the kinetic energy of the induced flow. We then analyze how closely this surrogate approximates the exact cost. We prove second-order agreement in a local commuting regime and derive an explicit cubic error bound in the local commuting regime. To handle nonlocal regimes, we introduce a path-splitting strategy that localizes the covariance evolution and enables piecewise application of the bound. We finally compare the surrogate with an exact construction based on the Gaussian Wasserstein geodesic and summarize the results in a practical regime map showing when the surrogate is accurate and the exact method is preferable.
\end{abstract}

\section{Introduction}

Transporting a probability distribution to another distribution
through a time-dependent density and velocity field
is a central problem in optimal transport, Schr\"odinger bridge methods,
diffusion models, and flow matching
\cite{villani2009optimal,ho2020denoising,song2021score,shi2023dsbm,haviv2025wasserstein}. 
In the dynamic formulation of quadratic OT, transport is expressed as a minimum-action problem over density--velocity pairs satisfying the continuity equation \cite{benamou2000}, 
and is closely related to displacement interpolation in Wasserstein space \cite{mccann1997}. 
A closely related stochastic counterpart is the Schr\"odinger bridge, which can be interpreted from the viewpoints of entropic optimal transport and stochastic control \cite{chen2016}. 
In parallel, recent flow matching approaches learn continuous-time vector fields along prescribed probability paths and have become an important tool in generative modeling \cite{lipman2023,tong2024}.

For Gaussian measures, the $2$-Wasserstein geometry admits an explicit form, with the covariance term governed by the Bures--Wasserstein geometry on positive definite matrices \cite{takatsu2011,bhatia2019}. 
This analytic structure makes Gaussian transport particularly attractive as a building block for mixture models. 
When the endpoint distributions are represented as Gaussian mixture models (GMMs), transport mechanisms can therefore be constructed component-wise and then assembled into a global mixture flow.

A particularly relevant recent work is \emph{Go With the Flow: Fast Diffusion for Gaussian Mixture Models} \cite{rapakoulias2025}. 
In that work, the endpoint distributions are represented as GMMs and transport is constructed using an analytic parametrization of feasible policies for mixture endpoints. 
The global policy is obtained through an optimization over component couplings. 
More broadly, recent Schr\"odinger bridge and flow-matching approaches have investigated simulation-free methods for modeling transport between complex endpoint distributions \cite{tong2024,gushchin2024}.

The present paper follows a different analytic route. 
Instead of starting from conditional Schr\"odinger bridges, we begin from deterministic Gaussian paths and derive the affine velocity field that satisfies the continuity equation for those paths. 
This construction yields a fully closed-form pairwise kinetic transport cost $C$ corresponding to linear interpolation in means and covariances. 
The central question is then geometric: how close is this surrogate cost $C$ to the exact Gaussian squared Wasserstein cost $W_2^2$?

We consider transport between two Gaussian mixture models and develop a continuity-equation-based construction of a time-dependent density and velocity field connecting the prescribed source and target endpoints. The precise mathematical formulation is deferred to the next section.


The main contribution of this paper is a closed-form transport surrogate for Gaussian mixture models together with explicit accuracy guarantees and a regime map for selecting between approximate and exact optimal transport geometry. 
More specifically, we report two analytic transport constructions between GMMs: a surrogate-based method built from affine Gaussian velocity fields and a second method based on the exact Gaussian Wasserstein geodesic. 
We then quantify the approximation quality of the surrogate construction through a local commuting second-order comparison, an explicit cubic gap bound in the local commuting regime, and a path-splitting strategy enabling piecewise local analysis in nonlocal regimes. Our contributions are as follows:

\begin{enumerate}
\item We develop a training-free, closed-form surrogate for Gaussian component transport using affine flow dynamics with linearly interpolated means and covariances.
\item We establish that this surrogate matches the exact Gaussian squared Wasserstein cost up to second order under a local commuting assumption.
\item We derive an explicit cubic bound on the surrogate--Wasserstein gap in the local commuting regime.
\item We introduce a path-splitting strategy that enables piecewise local error control in nonlocal regimes.
\item We present numerical and computational evidence clarifying the efficiency--accuracy trade-off between the surrogate construction and exact Gaussian Wasserstein transport, and identifying the regimes in which each method should be preferred.
\end{enumerate}

\paragraph*{Organization of the paper}
Sections~II--V develop Method~A and its analysis: Section~II formulates the GMM transport problem and introduces the surrogate construction, Section~III compares the surrogate with the exact Gaussian squared Wasserstein cost, Section~IV establishes a local cubic gap bound, and Section~V extends the analysis to nonlocal regimes via path splitting. 
Section~VI presents Method~B based on the exact Gaussian Wasserstein geodesic. 
Sections~VII and VIII report the numerical results and discuss their practical implications. 
Section~IX concludes the paper.

\paragraph*{Notation}
The main notation used throughout the paper is summarized in Table~\ref{tab:notation_intro}. 
Throughout the theoretical analysis, all covariance matrices are assumed to belong to $S_{++}^d$. Hence they are orthogonally diagonalizable with strictly positive eigenvalues, which guarantees nondegenerate Gaussian densities and the well-posedness of matrix inverses, matrix square roots, and related spectral operations.

\begin{table}[!t]
\centering
\caption{Basic notation used throughout the paper.}
\label{tab:notation_intro}
\begin{tabular}{p{0.22\linewidth} p{0.68\linewidth}}
\toprule
\textbf{Symbol} & \textbf{Meaning} \\
\midrule
$\mathbb{R}^d$ & $d$-dimensional Euclidean space \\
$\mathcal{S}_{++}^d$ & set of $d\times d$ symmetric positive definite matrices 
$=\left\{\Sigma\in\mathbb{R}^{d\times d}\;\middle|\;\Sigma=\Sigma^\top,\;x^\top \Sigma x>0\ \forall x\neq 0\right\}$ \\
$\rho_0,\rho_1$ & source and target endpoint distributions \\
$\rho_t$ & time-dependent density \\
$u(t,x)$ & velocity field at time $t$ and position $x$ \\
$\mathcal{N}(x|\mu,\Sigma)$ & Gaussian distribution, mean $\mu$ and covariance $\Sigma$ \\
$K_0,K_1$ & number of source and target mixture components \\
$a_i,b_j$ & source and target mixture weights \\
$\mu_i^0,\mu_j^1$ & source and target component means \\
$\Sigma_i^0,\Sigma_j^1$ & source and target component covariances \\
$\pi_{ij}$ & coupling weight between source component \(i\) and target component \(j\) \\
$C$ & surrogate transport cost \\
$W_2^2$ & squared $2$-Wasserstein distance \\
$\|\cdot\|$ & the spectral/ operator norm for matrices and the Euclidean norm for vectors. For symmetric matrices, $\|A\|=\max_i |\lambda_i(A)|$. \\
\bottomrule
\end{tabular}
\end{table}

\section{Problem Formulation and Method A}
Given two GMMs on $\mathbb{R}^d$,
\begin{equation}
\rho_0(x)=\sum_{i=1}^{K_0} a_i \mathcal{N}(x\mid \mu_i^0,\Sigma_i^0),
\
\rho_1(x)=\sum_{j=1}^{K_1} b_j \mathcal{N}(x\mid \mu_j^1,\Sigma_j^1)
\end{equation}
we aim to construct a transport dynamics, represented by a time-dependent density $\rho_t$ and a velocity field $u(t,x)$, that connects the prescribed source and target distributions.

\subsection{Component-wise Gaussian path and the affine field}

For each Gaussian component pair $(i,j)$, define the baseline path
\begin{equation}
\mu_{ij}(t)=(1-t)\mu_i^0+t\mu_j^1,\qquad
\Sigma_{ij}(t)=(1-t)\Sigma_i^0+t\Sigma_j^1.
\label{eq:linear-path}
\end{equation}

\begin{lemma}[Affine continuity-equation field]
\label{lem:unique-affine}
Let
\[
\rho_t(x)=\mathcal{N}(x\mid m(t),\Sigma(t)),
\qquad \Sigma(t)\in \mathbb{S}_{++}^d,
\]
where $m(\cdot)$ and $\Sigma(\cdot)$ are differentiable in $t$.
Consider affine vector fields of the form $v(t,x)=a(t)+B(t)\bigl(x-m(t)\bigr).$
Then the continuity equation $\partial_t \rho_t+\nabla\!\cdot(\rho_t v)=0$
admits an affine solution, namely
\begin{equation}
v(t,x)=\dot m(t)+\tfrac12\,\dot\Sigma(t)\Sigma(t)^{-1}\bigl(x-m(t)\bigr).
\label{eq:vector-field}
\end{equation}
Equivalently,
\begin{equation}
a(t)=\dot m(t),
\qquad
B(t)=\tfrac12\,\dot\Sigma(t)\Sigma(t)^{-1}.
\label{eq:parameter_affine_field}
\end{equation}
\end{lemma}

\begin{proof}
The proof is given in Appendix~\ref{app:proof-affine-field}.
\end{proof}
Hence, for each pair $(i,j)$ we define
\begin{equation}
v_{ij}(t,x)=\dot\mu_{ij}(t)+\frac12 \dot\Sigma_{ij}(t)\Sigma_{ij}(t)^{-1}(x-\mu_{ij}(t)).
\label{eq:vij}
\end{equation}

\subsection{Closed-form kinetic cost for a pair $(i,j)$}

Let $\rho_{ij}(t)=\mathcal{N}(x\mid \mu_{ij}(t),\Sigma_{ij}(t))$ be the Gaussian path defined by linear interpolation, and let $v_{ij}$ be the associated affine velocity field.

\begin{lemma}[Closed-form surrogate cost (Gaussian pair)]
\label{lem:Cij-closed}
 The kinetic cost
\[
C_{ij} = \int_0^1 \mathbb{E}\|v_{ij}(t,X)\|^2 dt
\]
admits the closed form
\begin{equation}
\begin{split}
C_{ij}
&=\|\mu_j^1-\mu_i^0\|^2\\
&
+
\frac14
\Tr\!\Big(
(\Sigma_i^0)^{-1/2}
\phi(C_0)
(\Sigma_i^0)^{-1/2}
(\Delta\Sigma_{ij})^2
\Big),
\end{split}
\label{eq:Cij-closed}
\end{equation}
where $C_0=(\Sigma_i^0)^{-1/2}\Delta\Sigma_{ij}(\Sigma_i^0)^{-1/2}$, $\phi(z)=\frac{\log(1+z)}{z}$, $\phi(0)=1$ and $\Delta\Sigma_{ij}=\Sigma_j^1-\Sigma_i^0$. Since $\Sigma_i^0,\Sigma_j^1\in S_{++}^d$, we have
\[
I+C_0=(\Sigma_i^0)^{-1/2}\Sigma_j^1(\Sigma_i^0)^{-1/2}\in S_{++}^d.
\]
Therefore, every eigenvalue $\lambda$ of $C_0$ satisfies $1+\lambda>0$, and $\phi(C_0)$ is well defined by spectral calculus.
\end{lemma}

\noindent
\begin{proof}
The detailed proof are provided in Appendix~\ref{app:Cij-proof}.
\end{proof}

\subsection{Global coupling and mixture flow}
\label{subsec:global-flow}

Given the pairwise costs $C_{ij}$ and $\epsilon_{OT}$ (the Sinkhorn temperature), we compute the entropic
optimal transport coupling ($\text{s.t.} \ \
\pi\mathbf 1=a,\ \
\pi^\top\mathbf 1=b$)
\begin{equation}
\begin{aligned}
\pi^\star
&\in
\arg\min_{\pi\ge0}
\;\langle\pi,C\rangle
+\varepsilon_{\mathrm{OT}}
\sum_{i,j}\pi_{ij}(\log\pi_{ij}-1) \\
\end{aligned}
\label{eq:sinkhorn}
\end{equation}

The interpolating density is defined as the mixture
\begin{equation}
\rho_t(x)
=
\sum_{i,j}
\pi^\star_{ij}
\,
\mathcal N
\big(x\mid
\mu_{ij}(t),
\Sigma_{ij}(t)
\big).
\label{eq:mixture-rho}
\end{equation}

The global velocity field is obtained as a responsibility-weighted
combination of the component flows
\begin{equation}
\begin{aligned}
u(t,x)
&=
\sum_{i,j}
\gamma_{ij}(t,x)\,
v_{ij}(t,x),\\
\gamma_{ij}(t,x)
&=
\frac{
\pi^\star_{ij}
\mathcal N(x\mid\mu_{ij}(t),\Sigma_{ij}(t))
}{
\sum_{k,\ell}
\pi^\star_{k\ell}
\mathcal N(x\mid\mu_{k\ell}(t),\Sigma_{k\ell}(t))
}.
\end{aligned}
\label{eq:global-velocity}
\end{equation}

The procedure for constructing the transport flow is
summarized in Algorithm~\ref{alg:methodA}. The resulting velocity field $u(t,x)$ defines the transport
dynamics using the ODE ($\dot{x}(t)=u(t,x(t))$)
with solutions generating the transport trajectories.

\begin{algorithm}[t]
\caption{Construction of the GMM Transport Flow (Method A)}
\label{alg:methodA}
\begin{algorithmic}[1]

\STATE Fit or obtain the endpoint GMMs $\rho_0$ and $\rho_1$.

\FOR{each component pair $(i,j)$}

\STATE Construct the Gaussian path 
$(\mu_{ij}(t),\Sigma_{ij}(t))$
using \eqref{eq:linear-path}.

\STATE Compute the affine velocity field
$v_{ij}(t,x)$ using \eqref{eq:vij}.

\STATE Compute the pairwise surrogate cost
$C_{ij}$ using \eqref{eq:Cij-closed}.

\ENDFOR

\STATE Solve the entropic OT problem
\eqref{eq:sinkhorn} using Sinkhorn iterations
to obtain $\pi^\star$.

\STATE Define the global mixture flow using
\eqref{eq:mixture-rho}--\eqref{eq:global-velocity}.

\end{algorithmic}
\end{algorithm}

\section{How Close Is $C$ to the Wasserstein Cost?}
\label{sec:comparison}

For Gaussian measures, the squared $2$-Wasserstein distance between
$\mathcal{N}(\mu_i^0,\Sigma_i^0)$ and $\mathcal{N}(\mu_j^1,\Sigma_j^1)$ is
\cite{takatsu2011,bhatia2019}
\begin{equation}
\begin{aligned}
W_2^2
\Big(
\mathcal{N}(\mu_i^0,\Sigma_i^0),
\mathcal{N}(\mu_j^1,\Sigma_j^1)
\Big)
=
\|\mu_j^1-\mu_i^0\|^2
+&
\Tr(\Sigma_i^0)
+
&\\\Tr(\Sigma_j^1)
-2\Tr\!\Big(
\big((\Sigma_i^0)^{1/2}\Sigma_j^1(\Sigma_i^0)^{1/2}\big)^{1/2}
\Big).
\end{aligned}
\label{eq:W2-gauss}
\end{equation}

We compare this exact Gaussian transport cost with the surrogate cost
$C_{ij}$ in \eqref{eq:Cij-closed}. Since both contain the same mean term,
the comparison reduces to the covariance contribution. 

Define the whitened
perturbation size
\begin{equation}
\hat{\rho} := \|C_0\|,
\qquad
C_0=(\Sigma_i^0)^{-1/2}\Delta\Sigma_{ij}(\Sigma_i^0)^{-1/2}.
\label{eq:rho-def}
\end{equation}

The parameter $\hat{\rho}$ quantifies the local size of the covariance perturbation.
When $\hat{\rho}<1$, both costs admit local analytic expansions.

\begin{theorem}[Local agreement of $C_{ij}$ and $W_{2,ij}^2$]
\label{thm:local-comparison}
Assume $\hat{\rho}<1$. Then the following hold. Let $\Delta\mu_{ij}:=\mu_j^1-\mu_i^0$ and $\Delta\Sigma_{ij}:=\Sigma_j^1-
\Sigma_i^0$.

\begin{enumerate}
\item
Both $C_{ij}$ and $W_{2,ij}^2$ contain the identical mean contribution
$\|\Delta\mu_{ij}\|^2$.
\item
The surrogate cost admits the expansion
\begin{equation}
\begin{aligned}
C_{ij}
=
\|\Delta\mu_{ij}\|^2
+&
\frac14
\Tr\!\big(
\Delta\Sigma_{ij}(\Sigma_i^0)^{-1}\Delta\Sigma_{ij}
\big)&\\
+
O(\|C_0\|^3).    
\end{aligned}
\label{eq:C-second-order}
\end{equation}

\item
If $\Sigma_i^0$ and $\Delta\Sigma_{ij}$ commute, then the Gaussian
Wasserstein cost satisfies
\begin{equation}
\begin{aligned}
W_{2,ij}^2
=
\|\Delta\mu_{ij}\|^2
+&
\frac14
\Tr\!\big(
\Delta\Sigma_{ij}(\Sigma_i^0)^{-1}\Delta\Sigma_{ij}
\big)&\\
+
O(\|C_0\|^3).    
\end{aligned}
\label{eq:W2-second-order}
\end{equation}
\end{enumerate}

Hence, in the local commuting regime, $C_{ij}-W_{2,ij}^2 = O(\|C_0\|^3)$.
In particular, the two costs agree through second order, and any discrepancy
appears only at cubic order and beyond.
\end{theorem}

\begin{proof}
The details of proof are given in
Appendix~\ref{app:local-expansions}.
\end{proof}

\begin{remark}
A more general local second-order expansion of the Gaussian
Bures--Wasserstein covariance term can also be derived via a
Fr\'echet expansion of the matrix square root together with the
associated Sylvester equation. In the commuting case, this
general expansion reduces to the closed-form quadratic proxy which is the specialization stated in Theorem~\ref{thm:local-comparison}. We report
this simpler form here due to space limitations.
\end{remark}

\section{Explicit Cubic Gap Bound Under $\hat{\rho}<1$}
\label{sec:cubic-gap}

We now quantify the discrepancy between the surrogate cost $C_{ij}$ and the exact Gaussian Wasserstein cost $W_{2,ij}^2$. Let $m_0:=\lambda_{\min}(\Sigma_i^0),\ \
M_0:=\lambda_{\max}(\Sigma_i^0),\ \
\kappa:=\frac{M_0}{m_0}$
and recall the whitened perturbation size \eqref{eq:rho-def}. Under the locality condition $\hat{\rho}<1$, the local expansions used to compare $C_{ij}$ and $W_2^2$ are valid.

\begin{theorem}[Explicit cubic gap bound]
\label{thm:cubic-gap}
Let $\Sigma_i^0,\Sigma_j^1\in\mathbb{S}_{++}^d$ and
$\Delta\Sigma_{ij}:=\Sigma_j^1-\Sigma_i^0$. Assume the local ($\hat{\rho}<1$) commuting regime. Then
\begin{equation}
|C_{ij}-W_{2,ij}^2|
\le
\bigl(
B_C(\Sigma_i^0,\hat{\rho})+B_W(\Sigma_i^0,\hat{\rho})
\bigr)
\|\Delta\Sigma_{ij}\|^3,
\label{eq:cubic-main}
\end{equation}
where
\begin{equation}
\begin{aligned}
B_C&(\Sigma_i^0,\hat{\rho})
=
\frac{3d}{4}\,
\frac{M_0}{m_0^3(1-\hat{\rho})},
\\&
B_W(\Sigma_i^0,\hat{\rho})
=
16\sqrt{2}\,d\,
\frac{M_0^6(1+\hat{\rho})^{3/2}}
     {m_0^8(1-\hat{\rho})^4}.
\end{aligned}
\label{eq:BCBW}
\end{equation}
\end{theorem}

\begin{proof}
In the local ($\hat{\rho}<1$) commuting regime, $Q_{ij}$ denotes the common quadratic proxy from Theorem~\ref{thm:local-comparison}. 
\[Q_{ij}
=
\frac14
\Tr\!\big(
\Delta\Sigma_{ij}(\Sigma_i^0)^{-1}\Delta\Sigma_{ij}
\big).\]

The discrepancy can therefore be decomposed as $W_{2,ij}^2-C_{ij}
=(W_{2,ij}^2-Q_{ij})-(C_{ij}-Q_{ij})$.
By applying triangle inequality:
\begin{equation}
|C_{ij}-W_{2,ij}^2|
\le
|C_{ij}-Q_{ij}|+|W_{2,ij}^2-Q_{ij}|.
\label{eq:gap-split}
\end{equation}

The first and second terms are respectively given by
\begin{equation}
|C_{ij}-Q_{ij}|
\le
B_C(\Sigma_i^0,\hat{\rho})\,\|\Delta\Sigma_{ij}\|^3.
\label{eq:C-bound-main}
\end{equation}
\begin{equation}
|W_{2,ij}^2-Q_{ij}|
\le
B_W(\Sigma_i^0,\hat{\rho})\,\|\Delta\Sigma_{ij}\|^3.
\label{eq:W-bound-main}
\end{equation}

Combining \eqref{eq:C-bound-main}, \eqref{eq:W-bound-main}, and \eqref{eq:gap-split} proves \eqref{eq:cubic-main}. Full derivations are provided in Appendix~\ref{app:cubic-derivation}.
\end{proof}

Theorem~\ref{thm:cubic-gap} makes explicit how the gap depends on both the perturbation size and the conditioning of the reference covariance. In particular, the bound deteriorates when $\Sigma_i^0$ becomes ill-conditioned ($m_0\to 0$).

\section{Piecewise-Local Transport via Path Splitting}
\label{sec:path-splitting}
The cubic gap bound of Theorem~\ref{thm:cubic-gap} requires both a locality condition and the assumptions stated in that theorem. In challenging regimes the locality condition may be violated, for instance when the covariance displacement is large or when the initial covariance is poorly conditioned. In such situations the local Taylor expansion underlying the bound no longer holds globally.

From a control perspective, path splitting can be viewed as replacing a single global transport policy with a sequence of piecewise-local transport policies, each operating in a regime where the surrogate approximation remains accurate. To recover locality, we therefore subdivide the transport path into a sequence of smaller covariance increments. The idea is to ensure that each substep satisfies the locality requirement of Theorem~\ref{thm:cubic-gap}, so that the theorem can be applied segmentwise whenever its other assumptions are also satisfied. Path splitting addresses the locality requirement by replacing a large covariance displacement with smaller segmentwise increments, thereby enabling piecewise-local analysis whenever the remaining assumptions hold on each segment. Note that path splitting only addresses the locality requirement; it does not remove the commuting assumption required by Theorem~\ref{thm:cubic-gap}.

Let $(\mu^{(N)},\Sigma^{(N)})=(\mu_j^1,\Sigma_j^1)$ and define the uniformly subdivided covariance path
\begin{equation}
\Sigma^{(k)}=\Sigma_i^0+\frac{k}{N}\Delta\Sigma_{ij},
\ \ \
k=0,\dots,N,
\ \ \
\Delta\Sigma_{ij}:=\Sigma_j^1-\Sigma_i^0.
\label{eq:linear-splitting}
\end{equation}
Since the cone of positive definite matrices is convex, each $\Sigma^{(k)}$ is positive definite. The increment on segment $k$ is then
\begin{equation}
\Delta\Sigma^{(k)}=\Sigma^{(k+1)}-\Sigma^{(k)}=\frac{1}{N}\Delta\Sigma_{ij}.
\label{eq:segment-increment}
\end{equation}
For each segment, define the local whitened perturbation
\begin{equation}
\hat{\rho}_k :=
\left\|
(\Sigma^{(k)})^{-1/2}
\Delta\Sigma^{(k)}
(\Sigma^{(k)})^{-1/2}
\right\|.
\label{eq:rhok}
\end{equation}

\begin{proposition}[Local validity under subdivision]
\label{prop:path-subdivision}
Assume that the subdivided path remains in $\mathbb{S}_{++}^d$. If, for every
$k=0,\dots,N-1$, the local perturbation satisfies $\hat{\rho}_k<1$, then the cubic
bound of Theorem~\ref{thm:cubic-gap} applies on each segment. Consequently,
\begin{equation}
\begin{aligned}
&\sum_{k=0}^{N-1}|C_{ij}(\Sigma^{(k)},\Sigma^{(k+1)})-W_{2,ij}^2(\Sigma^{(k)},\Sigma^{(k+1)})|
\le\\&
\frac{\|\Delta\Sigma_{ij}\|^3}{N^3}\sum_{k=0}^{N-1}
\Big(
B_C(\Sigma^{(k)},\hat{\rho}_k)
+
B_W(\Sigma^{(k)},\hat{\rho}_k)
\Big)
 .
\end{aligned}
\label{eq:global-sum}
\end{equation}
If the path is subdivided finely enough, then the segmentwise surrogate error relative to \(W_2^2\) over the entire subdivision is cumulatively controlled by the sum of the local cubic bounds. In this sense, the global error along the subdivided path is controlled through the accumulation of local errors, providing a principled way to extend the surrogate beyond its one-shot local validity regime.
\end{proposition}

\begin{proof}
Theorem~\ref{thm:cubic-gap} applies to each segment
$(\Sigma^{(k)},\Sigma^{(k+1)})$ whenever $\hat{\rho}_k<1$. Summing the resulting
segmentwise bounds yields \eqref{eq:global-sum}.
\end{proof}

A simple sufficient condition can be obtained using the submultiplicative property of the norm and \eqref{eq:segment-increment}-\eqref{eq:rhok}:
\begin{equation}
\begin{aligned}
\hat{\rho}_k
\le
\frac{1}{N}\,
\|(\Sigma^{(k)})^{-1/2}\|^2
\|\Delta\Sigma_{ij}\|.    
\end{aligned}
\label{eq:rhok-bound}
\end{equation}

Hence, whenever one has a uniform lower bound $\lambda_{\min}(\Sigma^{(k)}) \ge \underline m >0
\ \ \text{for all } k$,
it follows that
\begin{equation}
\hat{\rho}_k \le \frac{\|\Delta\Sigma_{ij}\|}{N\,\underline m}.
\label{eq:rhok-bound-uniform}
\end{equation}
where $m$ denotes a uniform lower bound on the smallest eigenvalues along the subdivided path. Therefore, choosing
\begin{equation}
N>\frac{\|\Delta\Sigma_{ij}\|}{\underline m}
\label{eq:N-condition}
\end{equation}
guarantees $\hat{\rho}_k<1$ for all segments. In practice, this shows that even when the original perturbation is too large
for the one-shot local analysis, a sufficiently fine subdivision restores the
local regime on each segment.

\section{Method B: Exact Gaussian Wasserstein Transport}
\label{sec:methodB}

As a reference construction, we also consider the exact Gaussian
$2$-Wasserstein geometry. In contrast to Method~A, which is based on the
surrogate cost $C_{ij}$, Method~B follows the true displacement interpolation
induced by the optimal Gaussian transport map
\cite{mccann1997,takatsu2011,bhatia2019}.

For the Gaussian pair
$\mathcal{N}(x\mid \mu_i^0,\Sigma_i^0)$ and $\mathcal{N}(x\mid \mu_j^1,\Sigma_j^1)$,
the optimal transport map is affine:
\begin{equation}
\begin{aligned}
T_{ij}(x)
&=
\mu_j^1 + M_{ij}(x-\mu_i^0),\\
M_{ij}
&=
(\Sigma_i^0)^{-1/2}
\Big(
(\Sigma_i^0)^{1/2}\Sigma_j^1(\Sigma_i^0)^{1/2}
\Big)^{1/2}
(\Sigma_i^0)^{-1/2}.
\end{aligned}
\label{eq:ot-map}
\end{equation}
The associated displacement interpolation has linear mean path $\mu_{ij}(t)=(1-t)\mu_i^0+t\mu_j^1$, and covariance geodesic
\begin{equation}
\Sigma_{ij}(t)=A_{ij}(t)\Sigma_i^0A_{ij}(t)^\top,
\ \
A_{ij}(t)=(1-t)I+tM_{ij}.
\label{eq:geo-path}
\end{equation}

The corresponding affine velocity field is
\begin{equation}
v_{ij}(t,x)
=
\dot\mu_{ij}(t)
+
\dot A_{ij}(t)A_{ij}(t)^{-1}(x-\mu_{ij}(t)).
\label{eq:vij-exact}
\end{equation}
This field satisfies the continuity equation and transports the Gaussian
density along the Wasserstein geodesic.

\begin{theorem}[Method~B yields the exact cost]
\label{thm:methodB}
For each Gaussian component pair, the kinetic action induced by
\eqref{eq:geo-path}--\eqref{eq:vij-exact} equals the exact squared Gaussian
Wasserstein distance:
\begin{equation}
C_{ij}^{(\mathrm{B})}
=
W_{2,ij}^2\!\Big(
\mathcal{N}(x\mid \mu_i^0,\Sigma_i^0),
\mathcal{N}(x\mid \mu_j^1,\Sigma_j^1)
\Big).
\label{eq:methodB-exact}
\end{equation}
\end{theorem}

\begin{proof}
The proof is provided in
Appendix~\ref{app:methodB-proof}.
\end{proof}

\section{Numerical Experiments}

\subsection{Accuracy Across Gaussian Regimes}

To assess the accuracy of the surrogate cost $C$, we compare it with the exact Wasserstein cost $W_2^2$ over the representative Gaussian scenarios reported in Table~\ref{tab:full11}, ranging from commuting cases to strongly non-commuting and ill-conditioned regimes. We quantify noncommutativity through the normalized commutator
\[
\mathrm{comm}:=
\frac{\|[\Sigma_0,\Delta\Sigma]\|_F}{\|\Sigma_0\|_F\,\|\Delta\Sigma\|_F}
=
\frac{\|\Sigma_0\Delta\Sigma-\Delta\Sigma\Sigma_0\|_F}{\|\Sigma_0\|_F\,\|\Delta\Sigma\|_F}.
\]

Three regimes emerge from the table. First, in commuting or nearly commuting settings, $C$ and $W_2^2$ agree closely, in line with the second-order comparison of Section~\ref{sec:comparison}. Second, for moderately non-commuting covariances, the gap becomes positive but remains controlled. Third, in strongly nonlocal or ill-conditioned regimes, the discrepancy can grow substantially. The \textbf{near-SPD-boundary} example is the clearest illustration: large covariance displacement and poor conditioning produce a dramatic increase in the gap.

The table also shows that the gap is primarily driven by covariance geometry rather than mean separation. Indeed, the \textbf{2D non-comm} and \textbf{Non-comm + mean shift} cases have essentially the same gap, even though their total transport costs differ greatly. Overall, the numerics support the theoretical regime picture: the dominant factors are the covariance displacement $\|\Delta\Sigma\|$ and the conditioning through the extremal eigenvalues, while the ambient dimension appears to play a secondary role. This is consistent with the local cubic analysis of Section~\ref{sec:cubic-gap} and with the piecewise-local interpretation of Section~\ref{sec:path-splitting}.
\begin{table*}[t]
\centering
\caption{Exact costs and approximation errors across $10$ Gaussian scenarios.}
\label{tab:full11}
\scriptsize
\setlength{\tabcolsep}{4pt}
\begin{tabular}{lcccccccccc}
\toprule
Name & dim & $\hat{\rho}$ & $\kappa$ & $\|\Delta\Sigma\|$ & comm & $W_2^2$ & $C$ & gap & $\left|W_{2}^2 - (\|\Delta\mu\|^2+Q)\right|$ & $\left|C - (\|\Delta\mu\|^2+Q)\right|$ \\
\midrule
1D (always comm) & 1 & 3.00 & 1.00 & 3.00 & 0.00 & 5.00 & 5.04 & 0.04 & 1.25 & 1.21 \\
2D isotropic $\Sigma_0=I$ & 2 & 1.00 & 1.00 & 1.00 & 0.00 & 2.81 & 2.81 & 0.00 & 0.03 & 0.03 \\
2D comm (diag) & 2 & 1.00 & 9.00 & 5.00 & 0.00 & 3.17 & 3.19 & 0.02 & 0.23 & 0.24 \\
2D non-comm & 2 & 1.68 & 9.00 & 5.76 & 0.19 & 3.88 & 4.11 & 0.23 & 0.08 & 0.31 \\
Non-comm + mean shift & 2 & 1.68 & 9.00 & 5.76 & 0.19 & 370.88 & 371.11 & 0.23 & 0.08 & 0.31 \\
Near-SPD boundary & 2 & 93.05 & 1000.00 & 333.31 & 0.91 & 98.52 & 2363.09 & 2264.57 & 21791.60 & 19527.03 \\
$\kappa=10^6$ stress test & 5 & 0.98 & 1000000.00 & 1.28 & 0.43 & 0.99 & 1.79 & 0.81 & 0.35 & 1.16 \\
10D Toeplitz model & 10 & 24.57 & 75.17 & 5.25 & 0.20 & 7.06 & 9.23 & 2.16 & 21.04 & 18.87 \\
20D factor model & 20 & 100.38 & 101.38 & 20.39 & 0.06 & 41.11 & 60.10 & 18.99 & 548.34 & 529.35 \\
30D Wishart model & 30 & 6.10 & 12.49 & 2.10 & 0.09 & 6.71 & 7.91 & 1.19 & 2.77 & 1.58 \\
\bottomrule
\end{tabular}
\end{table*}

\subsection{Computational Complexity and Runtime Scaling}

Beyond approximation accuracy, a main motivation for introducing the surrogate cost $C$ is computational efficiency. To quantify this advantage, we study transport between two Gaussian mixture models (GMMs) with two source and two target components ($K_0=K_1=2$), while varying the ambient dimension from $d=2$ to $d=300$.

We consider three representative regimes: \textbf{1} (commuting mild), \textbf{2} (non-commuting balanced), and \textbf{3} (anisotropic stretched). Scenario~1 corresponds to a near-commuting setting where second-order agreement is expected to dominate, Scenario~2 captures moderate non-commutativity, and Scenario~3 represents a higher-anisotropy stress regime in which higher-order effects become more visible. Across dimensions, the main geometric indicators --- $\|\Delta\Sigma\|$, conditioning $\kappa$, non-commutativity level $\mathrm{comm}$, and the ratio parameter $\rho$ --- remain within comparable ranges; their summary statistics are reported in Table~\ref{tab:runtime_scenarios_summary}. For each dimension, these structural quantities are first computed as max-over-pairs and then averaged across dimensions.

For each dimension, we measure the runtime of the transport-construction pipeline: (i) pairwise cost computation, (ii) entropic OT (Sinkhorn) for the mixture coupling, (iii) construction of the pairwise velocity fields $v_{ij}(t,x)$, and (iv) evaluation of the global velocity field $u(t,x)$ through responsibility weights. The ODE integration step is intentionally excluded so that the reported runtime isolates the intrinsic cost of transport construction. The resulting scaling is shown in Fig.~\ref{fig:runtime-scaling}. We use $\varepsilon_{\mathrm{OT}}=5\times 10^{-2}$, a maximum of $5000$ Sinkhorn iterations, and a stopping tolerance of $10^{-9}$. All computations are performed on CPU without GPU acceleration.

\begin{figure}[t]
\centering
\includegraphics[width=0.9\linewidth]{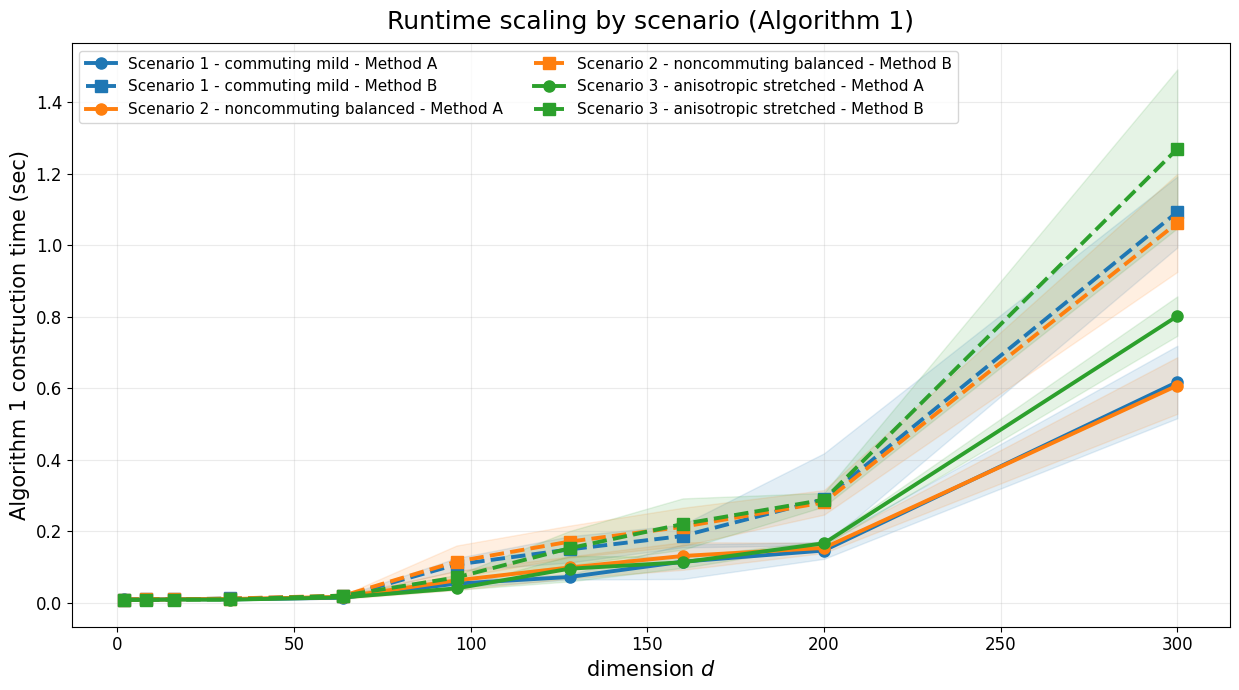}
\caption{Runtime scaling of the transport-construction pipeline for the surrogate and exact Wasserstein methods across three Gaussian regimes.}
\label{fig:runtime-scaling}
\end{figure}

\begin{table}[t]
\centering
\caption{Geometric indicators and mixture-level cost discrepancy.}
\label{tab:runtime_scenarios_summary}
\scriptsize
\setlength{\tabcolsep}{4pt}
\begin{tabular}{lccccc}
\toprule
Scenario & $\hat{\rho}$ & $\kappa$ & $\|\Delta\Sigma\|$ & comm & mean gap \\
\midrule
1 & 0.57 & 4.00  & 3.05  & 0.00 & 0.20 \\
2 & 0.47 & 10.00 & 4.05  & 0.19 & 0.02 \\
3 & 3.12 & 43.64 & 10.57 & 0.34 & 0.85 \\
\bottomrule
\end{tabular}
\end{table}
As shown in Fig.~\ref{fig:runtime-scaling}, both methods have similar cost in low dimension, but their runtimes separate as $d$ increases. Beyond roughly $d\approx 200$, the total construction time increases substantially in all three scenarios. This growth is consistently more pronounced for the Wasserstein-based method, whereas the surrogate pipeline remains systematically cheaper across the full tested range. This gap is explained by the underlying matrix operations. The surrogate construction relies on closed-form Gaussian interpolation together with matrix inversions and linear covariance operations. By contrast, the exact Wasserstein construction requires matrix square roots and Bures-type covariance transforms, whose cost becomes more significant in higher dimension.

Importantly, this computational advantage is not obtained at the price of a large loss in fidelity. Table~\ref{tab:runtime_scenarios_summary} shows that the mean absolute discrepancy between the mixture-level costs induced by $C$ and by $W_2^2$ remains moderate across all three regimes, including the anisotropic and non-commuting cases. Overall, these results indicate a favorable trade-off: the surrogate cost $C$ delivers a clear runtime advantage while remaining close to the exact Wasserstein cost at the mixture level, attractive in repeated-transport settings such as mixture-based generative modeling and iterative transport procedures.

\section{PRACTICAL REGIME MAP}

This paper considers two transport constructions only: Method~A, which uses the closed-form surrogate cost $C$, and Method~B, which follows the exact Gaussian Wasserstein geodesic. Method~A is appropriate in local and moderately conditioned regimes, where the covariance displacement is small enough for the surrogate approximation to remain accurate. Method~B is the exact alternative and is preferable in highly nonlocal, ill-conditioned, or near-boundary regimes.

Path splitting is not a separate transport construction. It is used only as a proof device for Method~A: when the global perturbation is too large for the local analysis, the covariance path can be subdivided into smaller segments so that the local bounds remain applicable on each piece.





\section{Conclusion}



We proposed a training-free method for Gaussian mixture transport based on affine Gaussian velocity fields and a closed-form surrogate cost $C$. 
The theory shows that $C$ approximates the Gaussian Bures--Wasserstein cost $W_2^2$ well in local and well-conditioned regimes, with larger errors mainly appearing for large covariance changes or ill-conditioned cases. 
Experiments confirm that the method is accurate in simple regimes and significantly faster than the exact Wasserstein construction, especially in high dimensions. 
A main future direction is to study the effect of the entropic OT parameter $\varepsilon_{\mathrm{OT}}$, which controls the trade-off between transport fidelity and smoothness.
\section*{Acknowledgment}
This work was supported by the European Union’s Horizon Europe research and innovation programme under the
Marie Skłodowska-Curie COFUND grant agreement No 101127936 (DeMythif.AI)
and Université Paris-Saclay through the PhD funding program. LLMs were used to help in editing the paper.

\appendix

\subsection{Proof of Lemma~\ref{lem:unique-affine}}
\label{app:proof-affine-field}

Let $\rho_t(x)=\mathcal{N}(x\mid m(t),\Sigma(t))$,
$y:=x-m(t)$
and consider an affine field, $
v(t,x)=a(t)+B(t)y.
$
Dividing the continuity equation
$\partial_t \rho_t+\nabla\!\cdot(\rho_t v)=0
$
by $\rho_t$ gives
\begin{equation}
\partial_t \log \rho_t+\nabla\!\cdot v+\langle \nabla \log \rho_t,v\rangle=0.
\label{eq:app-ce-log}
\end{equation}

For the Gaussian log-density,
$\log\rho_t(x)
=
-\frac d2\log(2\pi)
-\frac12\log\det\Sigma(t)
-\frac12 y^\top \Sigma(t)^{-1} y,$
based on $\nabla(x^TAx)=(A+A^T)x$, we have
\begin{equation}
\nabla \log \rho_t(x)=-\Sigma(t)^{-1}y,
\qquad
\nabla\!\cdot v(t,x)=\Tr(B(t)),
\label{eq:app-grad-div}
\end{equation}
\begin{equation}
\partial_t \log \rho_t(x)
=
-\frac12 \Tr\!\bigl(\Sigma^{-1}\dot\Sigma\bigr)
+\dot m^\top \Sigma^{-1} y
+\frac12 y^\top \Sigma^{-1}\dot\Sigma\,\Sigma^{-1} y.
\label{eq:app-dt-logrho-short}
\end{equation}

Substituting \eqref{eq:app-grad-div}--\eqref{eq:app-dt-logrho-short} into \eqref{eq:app-ce-log} yields
\begin{equation}
\begin{aligned}
-\frac12\Tr(\Sigma^{-1}\dot\Sigma)
+&\Tr(B)
+
(\dot m-a)^\top \Sigma^{-1} y
+\\&
y^\top\!\left(
\frac12 \Sigma^{-1}\dot\Sigma\,\Sigma^{-1}
-\Sigma^{-1}B
\right)y
=0  
\end{aligned}
\end{equation}
for all $y\in\mathbb{R}^d$. Therefore, the linear and quadratic terms must vanish, by selecting $a=\dot m$ and $B=\frac12\,\dot\Sigma\,\Sigma^{-1}$.
With this choice, $\Tr(B)=\frac12\Tr(\dot\Sigma\,\Sigma^{-1})
=\frac12\Tr(\Sigma^{-1}\dot\Sigma),$
so the constant term also cancels. Hence $v(t,x)=\dot m(t)+\frac12\,\dot\Sigma(t)\Sigma(t)^{-1}(x-m(t))
$
satisfies the continuity equation.
\hfill $\square$

\subsection{Proof of Lemma~\ref{lem:Cij-closed}}
\label{app:Cij-proof}

For an affine field $v(t,x)=a(t)+B(t)(x-\mu_{ij}(t))$
and $X\sim\mathcal{N}(\mu_{ij}(t),\Sigma_{ij}(t))$, we have
\begin{equation}
\mathbb{E}\|v(t,X)\|^2
=
\|a(t)\|^2
+
\Tr\!\big(B(t)\Sigma_{ij}(t)B(t)^\top\big).
\label{eq:app-energy}
\end{equation}

Substituting the affine velocity field using \eqref{eq:parameter_affine_field},
yields
\begin{equation}
\mathbb{E}\|v_{ij}(t,X)\|^2
=
\|\dot\mu_{ij}(t)\|^2
+
\frac14
\Tr\!\big(
\dot\Sigma_{ij}(t)\Sigma_{ij}(t)^{-1}\dot\Sigma_{ij}(t)
\big).
\label{eq:app-energy2}
\end{equation}

Under linear interpolation, it holds that $\dot\mu_{ij}(t)=\mu_j^1-\mu_i^0, 
\dot\Sigma_{ij}(t)=\Delta\Sigma_{ij}$, hence
\begin{equation}
C_{ij}
=
\|\mu_j^1-\mu_i^0\|^2
+
\frac14
\int_0^1
\Tr\!\Big(
\Delta\Sigma_{ij}\Sigma_{ij}(t)^{-1}\Delta\Sigma_{ij}
\Big)dt.
\label{eq:app-Cij}
\end{equation}

Let $C_0=(\Sigma_i^0)^{-1/2}\Delta\Sigma_{ij}(\Sigma_i^0)^{-1/2}.$
Then $\Sigma_{ij}(t)
=
(\Sigma_i^0)^{1/2}(I+tC_0)(\Sigma_i^0)^{1/2}$,
and $\Sigma_{ij}(t)^{-1}
=
(\Sigma_i^0)^{-1/2}(I+tC_0)^{-1}(\Sigma_i^0)^{-1/2}.$

Therefore,
\[
\int_0^1 \Sigma_{ij}(t)^{-1}dt
=
(\Sigma_i^0)^{-1/2}
\left[
\int_0^1 (I+tC_0)^{-1}dt
\right]
(\Sigma_i^0)^{-1/2}.
\]

By spectral calculus for symmetric matrices,
\[
\int_0^1 (I+tC_0)^{-1}dt
=
\phi(C_0),
\qquad
\phi(z)=\frac{\log(1+z)}{z}.
\]

Substituting into \eqref{eq:app-Cij} yields \eqref{eq:Cij-closed}.
\hfill $\square$

\subsection{Local expansions for $C$ and $W_2^2$}
\label{app:local-expansions}

We summarize the compact derivations used in
Theorem~\ref{thm:local-comparison}.

\subsubsection{Expansion of the surrogate cost}

Recall $C_{ij}$ \eqref{eq:Cij-closed} and then, using $\Delta\Sigma_{ij}=(\Sigma_i^0)^{1/2}C_0(\Sigma_i^0)^{1/2}$, it holds that
\[
\begin{aligned}
&(\Sigma_i^0)^{-1/2}
\Delta\Sigma_{ij} \Delta\Sigma_{ij} (\Sigma_i^0)^{-1/2}=\\
&(\Sigma_i^0)^{-1/2}((\Sigma_i^0)^{1/2}C_0(\Sigma_i^0)^{1/2})((\Sigma_i^0)^{1/2}C_0(\Sigma_i^0)^{1/2})(\Sigma_i^0)^{-1/2}\\
&=C_0 \Sigma_i^0 C_0
\end{aligned}
\]
Plugging into the trace and using cyclicity,
\begin{equation}
\begin{aligned}
&\Tr\!\Big(
(\Sigma_i^0)^{-1/2}\phi(C_0)(\Sigma_i^0)^{-1/2}(\Delta\Sigma_{ij})^2
\Big)\\
&= \Tr\!\Big(\phi(C_0)C_0\Sigma_i^0 C_0\Big)  
\end{aligned}
\label{eq:simple_c}
\end{equation}

For $\|C_0\|<1$, we have $\phi(C_0)=I-\tfrac12 C_0+O(\|C_0\|^2).$

Substituting $\phi(C_0)$ in \eqref{eq:simple_c} gives
\begin{equation}
\begin{aligned}
C_{ij}
&=
\|\Delta\mu_{ij}\|^2
+
\frac14
\Tr\!\big(C_0\Sigma_i^0 C_0\big)
+
O(\|C_0\|^3).
\end{aligned}
\label{eq:C-expand-app}
\end{equation}

Using \eqref{eq:rho-def}
and cyclicity of the trace,
\begin{equation}
\Tr(C_0\Sigma_i^0 C_0)
=
\Tr(\Delta\Sigma_{ij}\,(\Sigma_i^0)^{-1}\Delta\Sigma_{ij}).
\label{eq:trace-rewrite-app}
\end{equation}

Combining \eqref{eq:C-expand-app} and
\eqref{eq:trace-rewrite-app} yields
\eqref{eq:C-second-order}.
\hfill $\square$

\subsubsection{Expansion of the Wasserstein cost in the commuting case}

Assume now that $\Sigma_i^0$ and $\Delta\Sigma_{ij}$ commute.
Since both matrices are symmetric, they are simultaneously
diagonalizable. 
Using $\Sigma_j^1
=
\Sigma_i^0+\Delta\Sigma_{ij}
=
(\Sigma_i^0)^{1/2}(I+C_0)(\Sigma_i^0)^{1/2}$, we can rewrite the square root part of equation \eqref{eq:W2-gauss} as follows. 
\begin{equation}
\begin{aligned}
((\Sigma_i^0)&^{1/2}\Sigma_j^1(\Sigma_i^0)^{1/2})^{1/2}
= (\Sigma_i^0(I+C_0)\Sigma_i^0)^{1/2}\\&=(\Sigma_i^0)^{1/2}(I+C_0)^{1/2}(\Sigma_i^0)^{1/2}
\label{eq:sqrt-commuting}
\end{aligned}
\end{equation}
Taking traces gives $\Tr\!\Big(((\Sigma_i^0)^{1/2}\Sigma_j^1(\Sigma_i^0)^{1/2})^{1/2}\Big)=\Tr\!\Big(\Sigma_i^0(I+C_0)^{1/2}\Big)$.
Then using the scalar/matrix expansion
$(I+C_0)^{1/2}
=
I+\tfrac12 C_0-\tfrac18 C_0^2+O(\|C_0\|^3)$,
we have
\begin{equation}
\begin{aligned}
\Tr\!\Big(((\Sigma_i^0)&^{1/2}\Sigma_j^1(\Sigma_i^0)^{1/2})^{1/2}\Big)=\Tr\!\Big(\Sigma_i^0\Big)+\frac12\Tr\!\Big(\Sigma_i^0 C_0\Big)-\\&\frac18\Tr\!\Big(\Sigma_i^0 C_0^2\Big)+O(\|C_0\|^3),
\label{"eq:TEW22"}
\end{aligned}
\end{equation}

Using \eqref{"eq:TEW22"} and considering $\Tr(\Sigma_j^1)
=\Tr(\Sigma_i^0)+\Tr(\Delta\Sigma_{ij})
=\Tr(\Sigma_i^0)+\Tr(\Sigma_i^0 C_0)$
together, the constant and linear terms cancel in
\eqref{eq:W2-gauss}, and we obtain
\begin{equation}
W_{2,ij}^2
=
\|\Delta\mu_{ij}\|^2
+
\frac14\Tr(\Sigma_i^0 C_0^2)
+
O(\|C_0\|^3).
\label{eq:W2-expand-app}
\end{equation}

Finally, $\Tr(\Sigma_i^0 C_0^2)
=
\Tr(\Delta\Sigma_{ij}\,(\Sigma_i^0)^{-1}\Delta\Sigma_{ij}),$
which proves \eqref{eq:W2-second-order}.
\hfill $\square$
\subsection{Derivation of the Cubic Gap Bound}
\label{app:cubic-derivation}

This appendix provides the main steps behind Theorem~\ref{thm:cubic-gap}. We bound the two remainders $C_{ij}-Q_{ij}$ and $W_{2,ij}^2-Q_{ij}$.
separately.

\subsubsection{Bound for the surrogate remainder}

From the closed form of $C_{ij}$ (using \eqref{eq:simple_c} and \eqref{eq:C-expand-app}), its covariance contribution can be written as
\begin{equation}
C_{\mathrm{cov}}-Q_{ij}
=
\frac14
\Tr\!\big((\phi(C_0)-I)C_0\Sigma_i^0 C_0\big),
\label{eq:CminusQ-app}
\end{equation}
For $\hat{\rho}=\|C_0\|<1$, the power-series expansion of
\[
\phi(C_0)=\frac{\log(I+C_0)}{C_0}=\sum_{k=0}^\infty \frac{(-1)^k}{k+1}C_0^k
\]
implies 
\begin{equation}
\|\phi(C_0)-I\|
\le \sum_{k=1}^\infty \frac{\hat{\rho}^k}{k+1}\le
\frac{3\hat{\rho}}{1-\hat{\rho}}.
\label{eq:phi-bound-app}
\end{equation}
Therefore, using $|\Tr(A)|\le d\|A\|$ and submultiplicativity,
\begin{equation}
\begin{aligned}
|C_{\mathrm{cov}}-Q_{ij}|
&\le
\frac{d}{4}\,
\|\phi(C_0)-I\|\,\|C_0\|^2\,\|\Sigma_i^0\| \\
&\le
\frac{d}{4}\,
\frac{3\hat{\rho}}{1-\hat{\rho}}\,\hat{\rho}^2\,M_0.
\end{aligned}
\label{eq:Cbound-step-app}
\end{equation}
Finally, $\|C_0\|=\|(\Sigma_i^0)^{-1/2}\Delta\Sigma_{ij}(\Sigma_i^0)^{-1/2}\|
\le
\|(\Sigma_i^0)^{-1/2}\|^2\,\|\Delta\Sigma_{ij}\|
=
\frac{\|\Delta\Sigma_{ij}\|}{m_0}$. So
\begin{equation}
|C_{ij}-Q_{ij}|
\le
\frac{3d}{4}
\frac{M_0}{m_0^3(1-\hat{\rho})}
\|\Delta\Sigma_{ij}\|^3.
\label{eq:C-final-app}
\end{equation}
This gives the constant $B_C(\Sigma_i^0,\hat{\rho})$.
\hfill $\square$

\subsubsection{Bound for the Wasserstein remainder}

Define the segment $\Sigma(\alpha)=\Sigma_i^0+\alpha\Delta\Sigma_{ij},
\ \ \alpha\in[0,1]$
\label{eq:Sigma-alpha-app} and set
\begin{equation}
A(\alpha)
=
(\Sigma_i^0)^{1/2}\Sigma(\alpha)(\Sigma_i^0)^{1/2},
\ \
g(\alpha)=\Tr\!\big(\sqrt{A(\alpha)}\big).
\label{eq:A-g-app}
\end{equation}
Then the covariance part of the Wasserstein cost is
\begin{equation}
W_{2,\mathrm{cov}}^2(\alpha)
=
\Tr(\Sigma_i^0)+\Tr(\Sigma(\alpha))-2g(\alpha).
\label{eq:Wcov-alpha-app}
\end{equation}
Since $\Tr(\Sigma(\alpha))$ is affine in $\alpha$, all terms of order three and higher come from $g(\alpha)$. Taylor's theorem therefore gives
\begin{equation}
W_{2,\mathrm{cov}}^2-Q_{ij}
=
-\frac13 g^{(3)}(\xi)
\label{eq:W-remainder-app}
\end{equation}
for some $\xi\in(0,1)$.

To bound $g^{(3)}$, we use the Fr\'echet derivatives of the matrix
square root. The first derivative admits the resolvent
representation \cite{Higham2008}
\begin{equation}
D\sqrt{A}[E]
=
\frac{1}{2\pi i}
\oint_\Gamma
z^{1/2}(zI-A)^{-1}E(zI-A)^{-1}\,dz,
\label{eq:resolvent-app}
\end{equation}
and higher derivatives are obtained by repeated differentiation.
Applying this to $A(\alpha)$ and using the spectral enclosure of
the path yields a uniform bound on the third derivative
$D^3\sqrt{A(\alpha)}$ over $\alpha\in[0,1]$.

Choose $\Gamma$ to be the circle centered at $c=(a+b)/2$ with radius $R=b/2$.
This circle encloses $[a,b]$. Its distance to the interval is at least $\delta=a/2$, hence
$\|(zI-A)^{-1}\|\le 1/\delta = 2/a$ for all $z\in\Gamma$.
Also $|z|\le c+R = (a+b)/2 + b/2 \le b + a/2 \le 2b$; a coarse but simple bound is $|z|\le 2b$,
hence $|f(z)|=|\sqrt{z}|\le \sqrt{|z|}\le \sqrt{2b}.$
The contour length is $|\Gamma|=2\pi R=\pi b$.
Thus
\begin{equation}
\begin{aligned}
&\|D^3 f(A)\|
\le \frac{3!}{2\pi}\int_\Gamma|f(z)|\|(zI-A)^{-1}\|^4|dz|\\
&\le\frac{6}{2\pi}\,|\Gamma|\,\sup_{z\in\Gamma}|f(z)|\,\sup_{z\in\Gamma}\|(zI-A)^{-1}\|^4\\
&\le\frac{6}{2\pi}\,(\pi b)\,(\sqrt{2b})\,(2/a)^4.
\label{eq:apendix_int}
\end{aligned}
\end{equation}

This simplifies to $\|D^3 f(A)\|\le 48\sqrt{2}\, b^{3/2}/a^4$.

It remains to identify a uniform spectral interval for $A(\alpha)$. Since \eqref{eq:rho-def} we have $(1-\alpha\hat{\rho})I\preceq I+\alpha C_0\preceq(1+\alpha\hat{\rho})I$ and multiplying by $\Sigma_i^0$ on both sides
\begin{equation}
\begin{aligned}
&(1-\alpha\hat{\rho})(\Sigma_i^0)^2\le (\Sigma_i^0)^2+\alpha(\Sigma_i^0)^{1/2}\Delta\Sigma_{ij}(\Sigma_i^0)^{1/2}\\
&\le (1+\alpha\hat{\rho})(\Sigma_i^0)^2\\
&\rightarrow(1-\alpha\hat{\rho})(\Sigma_i^0)^2\le(\Sigma_i^0)^{1/2}(\Sigma_i^0+\alpha\Delta\Sigma_{ij})(\Sigma_i^0)^{1/2}\\
&\le (1+\alpha\hat{\rho})(\Sigma_i^0)^2\\
&\rightarrow(1-\alpha\hat{\rho})(\Sigma_i^0)^2\le A(\alpha)\le (1+\alpha\hat{\rho})(\Sigma_i^0)^2.    
\end{aligned}    
\end{equation}

Seeing that $(\Sigma_i^0)^2$ and $A(\alpha)$ are symmetric, by taking eigenvalues minima/maxima yields
\begin{equation}
\begin{aligned}
&a=(1-\alpha\hat{\rho})\lambda_{\min}((\Sigma_i^0)^2)\le \lambda_{\min}(A(\alpha))\le
\\&\lambda_{\max}(A(\alpha))\le (1+\alpha\hat{\rho})\lambda_{\max}((\Sigma_i^0)^2)=b.
\label{eq:pectral_enclosure}
\end{aligned}
\end{equation}
uniformly for $\alpha\in[0,1]$. Also, with $E=(\Sigma_i^0)^{1/2}\Delta\Sigma_{ij}(\Sigma_i^0)^{1/2}$,
submultiplicativity gives
\begin{equation}
\|E\|
\le
\|(\Sigma_i^0)^{1/2}\|^2\,\|\Delta\Sigma_{ij}\|
=
M_0\,\|\Delta\Sigma_{ij}\|.
\label{eq:E-bound-app}
\end{equation}

Combining \eqref{eq:apendix_int}, \eqref{eq:pectral_enclosure}, and \eqref{eq:E-bound-app}, we obtain
\begin{equation}
|g^{(3)}(\alpha)|
\le
48\sqrt{2}\,d\,
\frac{M_0^6(1+\hat{\rho})^{3/2}}{m_0^8(1-\hat{\rho})^4}
\|\Delta\Sigma_{ij}\|^3.
\label{eq:g3-bound-app}
\end{equation}
Substituting \eqref{eq:g3-bound-app} into \eqref{eq:W-remainder-app} yields
\begin{equation}
|W_{2,ij}^2-Q_{ij}|
\le
16\sqrt{2}\,d\,
\frac{M_0^6(1+\hat{\rho})^{3/2}}{m_0^8(1-\hat{\rho})^4}
\|\Delta\Sigma_{ij}\|^3,
\label{eq:W-final-app}
\end{equation}
which gives the constant $B_W(\Sigma_i^0,\hat{\rho})$.
\hfill $\square$

\subsection{Exact kinetic action of Method B}
\label{app:methodB-proof}

In this appendix we show that the kinetic action associated
with Method~B equals the Gaussian Wasserstein cost.

Let the optimal Gaussian map, the geodesic covariance path,
and the associated velocity field be defined in
\eqref{eq:ot-map}, \eqref{eq:geo-path}, and \eqref{eq:vij-exact}. For simplicity we omit the pair indices (i, j).
Let $X_t\sim\mathcal{N}(\mu_t,\Sigma_t)$ and $v(t,x)=\dot\mu_t+\dot A(t)A(t)^{-1}(x-\mu_t)$.
Using \eqref{eq:app-energy} we have
\begin{equation}
\mathbb{E}\|v(t,X_t)\|^2
=
\|\dot\mu_t\|^2
+
\operatorname{Tr}
\!\left(
\dot A(t)A(t)^{-1}\Sigma_tA(t)^{-T}\dot A(t)^\top
\right).
\label{eq:app-energy1}
\end{equation}

Since $\Sigma_t=A(t)\Sigma_0A(t)^\top$, we obtain $A(t)^{-1}\Sigma_tA(t)^{-T}=\Sigma_0$,
and therefore $\mathbb{E}\|v(t,X_t)\|^2
=\|\dot\mu_t\|^2
+\operatorname{Tr}\!\big(\dot A(t)\Sigma_0\dot A(t)^\top\big)$.

Because $\dot\mu_t=\mu_1-\mu_0$ and $\dot A(t)=M-I$ are
constant in $t$, the integrand does not depend on time.
Hence
\begin{equation}
\int_0^1
\mathbb{E}\|v(t,X_t)\|^2\,dt
=
\|\mu_1-\mu_0\|^2
+
\operatorname{Tr}\!\big((M-I)\Sigma_0(M-I)^\top\big)
\label{eq:app-energy3}
\end{equation}

Expanding the trace term yields $\operatorname{Tr}\!\big((M-I)\Sigma_0(M-I)^\top\big)
=\operatorname{Tr}(M\Sigma_0M^\top)
+\operatorname{Tr}(\Sigma_0)
-2\operatorname{Tr}(M\Sigma_0)$.

From the definition of the optimal Gaussian map, it follows that $M\Sigma_0M^\top=\Sigma_1$,
and $\operatorname{Tr}(M\Sigma_0)
= \operatorname{Tr}\!\left(
(\Sigma_0^{1/2}\Sigma_1\Sigma_0^{1/2})^{1/2}\right)$. Substituting into \eqref{eq:app-energy3} gives
\begin{equation}
\begin{aligned}
&\int_0^1
\mathbb{E}\|v(t,X_t)\|^2\,dt
=
\|\mu_1-\mu_0\|^2
+\\&
\operatorname{Tr}(\Sigma_0)
+
\operatorname{Tr}(\Sigma_1)
-
2\operatorname{Tr}
\!\left(
(\Sigma_0^{1/2}\Sigma_1\Sigma_0^{1/2})^{1/2}
\right),    
\end{aligned}
\end{equation}
which is exactly the Gaussian Wasserstein cost
\eqref{eq:W2-gauss}.

\bibliographystyle{IEEEtran}
\bibliography{refs}

\end{document}